\documentclass[journal]{IEEEtran}
\usepackage{cite}

\ifCLASSINFOpdf
\else
\fi
\usepackage{amsmath}
\usepackage{algorithmic}

\usepackage{amssymb,amsfonts}
\usepackage{graphicx}
\usepackage{xcolor}
\usepackage{amsthm}
\usepackage{algorithm}
\usepackage{multirow}

\newtheorem{theorem}{Theorem}

\newtheorem{definition}{Definition}

\begin{document}
%
\title{A Tool for Neural Network Global Robustness Certification and Training}
%
%
%

\author{
\IEEEauthorblockN{Zhilu Wang\IEEEauthorrefmark{1},
Yixuan Wang\IEEEauthorrefmark{1},
Feisi Fu\IEEEauthorrefmark{2},
Ruochen Jiao\IEEEauthorrefmark{1},
Chao Huang\IEEEauthorrefmark{3},
Wenchao Li\IEEEauthorrefmark{2},
Qi Zhu\IEEEauthorrefmark{1}}\\
\IEEEauthorblockA{\IEEEauthorrefmark{1}Department of Electrical and Computer Engineering, Northwestern University, Evanston, IL, USA\\
Emails: \{zhilu.wang, yixuanwang2024, ruochenjiao2024\}@u.northwestern.edu, qzhu@northwestern.edu}\\
\IEEEauthorblockA{\IEEEauthorrefmark{2}Department of Electrical and Computer Engineering, Boston University, Boston, MA, USA\\
Emails: \{fufeisi, wenchao\}@bu.edu}\\
\IEEEauthorblockA{\IEEEauthorrefmark{3}Department of Computer Science, University of Liverpool, Liverpool, UK\\
Email: chao.huang2@liverpool.ac.uk
}}
\maketitle

\begin{abstract}
With the increment of interest in leveraging machine learning technology in safety-critical systems, the robustness of neural networks under external disturbance receives more and more concerns. Global robustness is a robustness property defined on the entire input domain. And a certified globally robust network can ensure its robustness on any possible network input. However, the state-of-the-art global robustness certification algorithm can only certify networks with at most several thousand neurons. In this paper, we  propose the GPU-supported global robustness certification framework GROCET, which is more efficient than the previous optimization-based certification approach. Moreover, GROCET provides differentiable global robustness, which is leveraged in the training of globally robust neural networks.
\end{abstract}


%
\IEEEpeerreviewmaketitle

\section{Introduction}
While most works are focusing on evaluating and improving the local robustness of neural networks, the high runtime overhead introduced by local robustness certification limits its applicable scenarios. Global robustness can perfectly avoid this issue as it can be certified offline and applied to the entire input domain. \cite{zhilu_2022_date} is the first efficient global robustness certification algorithm that can handle neural networks with thousands of neurons.
This work is an extension of the neural network global robustness certification algorithm~\cite{zhilu_2022_date}. In this work, we still leverage the interleaving twin-network encoding (ITNE) structure from~\cite{zhilu_2022_date}. However, instead of formulating the global robustness certification as an optimization problem, we propose an upgraded version of the neural network global robustness certification algorithm by symbolic prorogation and branch-and-bound, inspired by $\alpha, \beta$-CROWN~\cite{beta_crown}. Substituting the optimization problem by symbolic propagation allows us to enable parallel computing on more powerful platforms (e.g., GPU) and makes it easy and efficient to extract gradient information to facilitate the training process to obtain globally robust neural networks. We denote this new certification algorithm as GROCET. And the optimization-based algorithm in~\cite{zhilu_2022_date} is denoted as MILP-ITNE.

In this manuscript, we present the GROCET algorithm and prove its soundness in the over-approximation of neural network global robustness. And we propose a globally robust network training technique that treats the certified global robustness as a regularization term.

\section{System Model}

We consider a neural network $\mathcal{F}$ that maps the input $x_{in}$ into the output $x_{out}$. As the robustness property is defined for each network output, without loss of generality, the network output is assumed to be 1-dimensional, i.e., $x_{out}\in \mathbb{R}$. The network $\mathcal{F}$ is considered in a graph view, where each vertex in the graph represents a layer $f$ and the edges in the graph represent the network hidden states $x_i$. In this work, we support neural networks containing the following types of layers.

\subsection{Single-input-single-output layers}

\paragraph{Linear layer}
\[
x_i = f(x_j) = W x_j + b
\]
The linear layer can be either a fully-connected layer or a convolution layer.

\paragraph{ReLU layer}
\[
x_i = f(x_j) = \max(x_j, 0)
\]

\paragraph{Maxpooling layer}
\[
x_i = f(x_j) = \text{maxpooling}(x_j)
\]

\subsection{Double-input-single-output layers}

\paragraph{Addition layer}
\[
x_i=f(x_j, x_k) = x_j + x_k
\]

\paragraph{Subtraction layer}
\[
x_i=f(x_j, x_k) = x_j - x_k
\]

\begin{definition}[Global Robustness]
    The neural network $\mathcal{F}$ is $(\delta,\varepsilon)$-globally robust in the entire input domain $X$ if for any input $\forall x \in X$ and any bounded perturbation $\forall x'\in \{x'\in X \mid \|x'-x\|\leq \delta \}$, the network output variation satisfies
    \[
    |\mathcal{F}(x') - \mathcal{F}(x)| \leq \varepsilon.
    \]
\end{definition}

\section{GROCET Certification Algorithm}

Leveraging the interleaving twin-network encoding (ITNE), we denote the distance between each hidden state $x_i$ and the corresponding perturbed hidden state $x_i'$ as $\Delta x_i$. Inspired by the $\alpha,\beta$-CROWN, we propose an algorithm to over-approximate the neural network output variation bound under disturbance, which consists of the symbolic propagation and branch-and-bound.

\subsection{Symbolic Propagation} \label{sec:symb_prop}

Denote $\mathcal{X}=\{\Delta x_{in}, \Delta x_{out}, \Delta x_i, \forall i\}$ as the set of neural network input perturbation, output variation and all hidden state distances. T

\begin{theorem}
    The network output variation can have a linear lower and upper bound w.r.t. the network input perturbation $\Delta x_{in}$:
    \begin{equation} \label{equ:final_lin_bound}
        \begin{aligned}
    \mathcal{F}(x_{in} + \Delta x_{in}) - \mathcal{F}(x_{in}) &\geq A \Delta x_{in} + b,\\
    \mathcal{F}(x_{in} + \Delta x_{in}) - \mathcal{F}(x_{in}) &\leq C \Delta x_{in} + d.
    \end{aligned}
    \end{equation}
\end{theorem}

\begin{proof}
    This can be proved by induction. Assume that the network output variation can be bounded by a linear combination of a subset of state distances $\mathcal{X}'\subseteq \mathcal{X}$:
    \begin{equation} \label{equ:mid_lin_bound}
        \begin{aligned}
            \mathcal{F}(x_{in} + \Delta x_{in}) - \mathcal{F}(x_{in}) &\geq \sum_{x_i\in \mathcal{X}'} A^{\mathcal{X}'}_i \Delta x_i + b^{\mathcal{X}'},\\
            \mathcal{F}(x_{in} + \Delta x_{in}) - \mathcal{F}(x_{in}) &\leq \sum_{x_i\in \mathcal{X}'} C^{\mathcal{X}'}_i \Delta x_i + d^{\mathcal{X}'}.
        \end{aligned}
    \end{equation}
    
    This assumption is true when $\mathcal{X}'=\{\Delta x_{out}\}$, where $A^{\mathcal{X}'}_{out} = C^{\mathcal{X}'}_{out} = 1$ and $b^{\mathcal{X}'} = d^{\mathcal{X}'} = 0$. 

    The goal is to substitute all $\Delta x_i\in \mathcal{X} \setminus \{x_{in}\}$ by $x_{in}$. 

    Starting from the output layer, the linear bounds can propagate backward to the network input. Each layer needs to be bounded by a lower and upper bound. For instance, for layer $x_i=f(x_j, x_k)$ ($x_k$ is optional), the layer output variation needs to be bounded by its lower bound $L_i(\Delta x_j, \Delta x_k)$ and upper bound $U_i(\Delta x_j, \Delta x_k)$:
    \begin{equation} \label{equ:layer_lin_bound}
        \begin{aligned}
     \Delta x_i \geq L_i(\Delta x_j, \Delta x_k) &= \Psi_j \Delta x_j + \Psi_k \Delta x_k + \lambda\\
    \Delta x_i \leq U_i(\Delta x_j, \Delta x_k) &= \Omega_j \Delta x_j + \Omega_k \Delta x_k + \mu
    \end{aligned}
    \end{equation}

    In this case, for any $\forall \mathcal{X}'$ that $\Delta x_i \in \mathcal{X}'$, we can have:
    \begin{equation}
    A^{\mathcal{X}'}_i \Delta x_i \geq \left(A^{\mathcal{X}'}_i\right)^+ L_i(\Delta x_j, \Delta x_k) + \left(A^{\mathcal{X}'}_i\right)^- U_i(\Delta x_j, \Delta x_k),
    \end{equation}
    where $\left(A^{\mathcal{X}'}_i\right)^+ = \max(A^{\mathcal{X}'}_i, 0)$ and $\left(A^{\mathcal{X}'}_i\right)^- = \min(A^{\mathcal{X}'}_i, 0)$. Similarly, we have the upper bound of $C^{\mathcal{X}'}_i \Delta x_i$:
    \begin{equation}
    C^{\mathcal{X}'}_i \Delta x_i \leq \left(C^{\mathcal{X}'}_i\right)^+ U_i(\Delta x_j, \Delta x_k) + \left(C^{\mathcal{X}'}_i\right)^- L_i(\Delta x_j, \Delta x_k).
    \end{equation}

    Therefore, given the linear bounds w.r.t. a set $\mathcal{X}'$ as in Eq.~\eqref{equ:mid_lin_bound}, and the linear bounds of $\Delta x_i$ as in Eq.~\eqref{equ:layer_lin_bound}, 
    the linear bounds of set $\mathcal{X}''=\{\Delta x_j, \Delta x_k\} \bigcup \mathcal{X}' \setminus \{\Delta x_i\}$ can be derived. Specifically, we have:
    \begin{equation}\label{equ:matrixOps}
    \begin{aligned}
        A^{\mathcal{X}''}_j &= A^{\mathcal{X}'}_j + \left(A^{\mathcal{X}'}_i\right)^+ \Psi_j + \left(A^{\mathcal{X}'}_i\right)^- \Omega_j,\\
        A^{\mathcal{X}''}_k &= A^{\mathcal{X}'}_k + \left(A^{\mathcal{X}'}_i\right)^+ \Psi_k + \left(A^{\mathcal{X}'}_i\right)^- \Omega_k,\\
        A^{\mathcal{X}''}_h &= A^{\mathcal{X}'}_h, \quad \forall h \neq i, j, k\\
        b^{\mathcal{X}''} &= b^{\mathcal{X}'} + \left(A^{\mathcal{X}'}_i\right)^+ \lambda + \left(A^{\mathcal{X}'}_i\right)^- \mu,\\
        C^{\mathcal{X}''}_j &= C^{\mathcal{X}'}_j + \left(C^{\mathcal{X}'}_i\right)^+ \Omega_j + \left(C^{\mathcal{X}'}_i\right)^- \Psi_j,\\
        C^{\mathcal{X}''}_k &= C^{\mathcal{X}'}_k + \left(C^{\mathcal{X}'}_i\right)^+ \Omega_k + \left(C^{\mathcal{X}'}_i\right)^- \Psi_k,\\
        C^{\mathcal{X}''}_h &= C^{\mathcal{X}'}_h, \quad \forall h \neq i, j, k\\
        d^{\mathcal{X}''} &= d^{\mathcal{X}'} + \left(C^{\mathcal{X}'}_i\right)^+ \mu + \left(C^{\mathcal{X}'}_i\right)^- \lambda,
    \end{aligned}
    \end{equation}

    After propagating backward through layer $x_i=f(x_j, x_k)$, the linear bounds w.r.t. $\mathcal{X}'$ will be converted to the linear bounds w.r.t. $\mathcal{X}''$ which does not contains layer output distance $\Delta x_i$ anymore. Therefore, after propagating backward through all layers, we can derive the linear bounds only related to the network input perturbation $\Delta x_{in}$, as in Eq.~\eqref{equ:final_lin_bound}.
\end{proof}

\subsection{Layer bounds in Eq.~\eqref{equ:layer_lin_bound}}

To derive the linear bounds in Eq.~\eqref{equ:final_lin_bound}, the linear bounds for each layer output w.r.t. its inputs, i.e., the bounds in Eq.~\eqref{equ:layer_lin_bound}, are needed.

\subsubsection{Linear layers}

For linear layer $x_i = f(x_j) = W x_j + b$, we have $\Delta x_i = W\Delta x_j$, which is already linear. Therefore, Eq.~\eqref{equ:layer_lin_bound} for linear layer can be expressed as:
\[
\psi_j=\omega_j=W,\quad \lambda=\mu=0.
\]

Similar expressions can also be found for addition layers and subtraction layers, since they are also linear. For addition layers, we have 
\[
\psi_j=\omega_j=\psi_k=\omega_k=1,\quad \lambda=\mu=0.
\]
And for subtraction layers, we have 
\[
\psi_j=\omega_j=1\quad \psi_k=\omega_k=-1,\quad \lambda=\mu=0.
\]

\begin{figure*}[htb]
  \centering
  \includegraphics[width=0.8\textwidth]{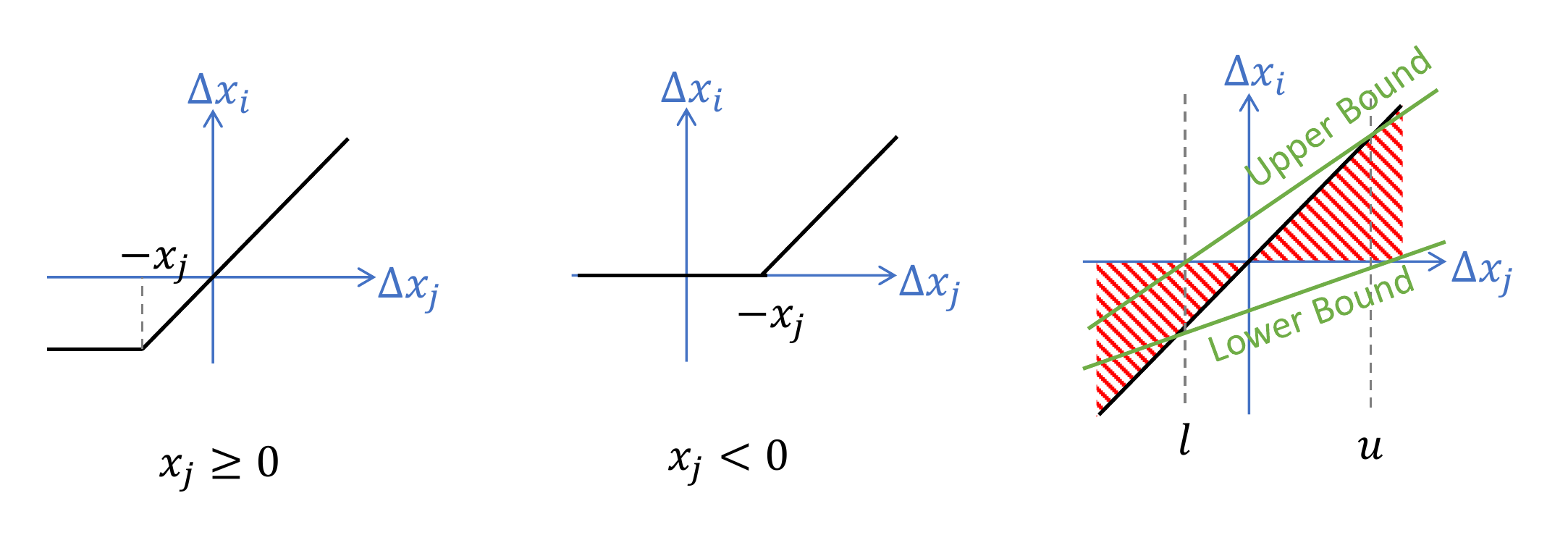}
  \caption{The relation between $\Delta x_i$ and $\Delta x_j$ for the ReLU distance $\Delta x_i=relu(x_j + \Delta x_j) - relu(x_j)$, which is related to the value of relu input $x_j$. This figure was originally presented in~\cite{zhilu_2022_date}.}
  \label{fig:relu_distance}
\end{figure*}

\subsubsection{ReLU layer}
For the ReLU layer $x_i = f(x_j) = \max(x_j, 0)$, $\Delta x_i = \max(x_j + \Delta x_j, 0) - \max(x_j, 0)$, which is not linear. Also, the value of $x_j$ will impact the relation between $\Delta x_i$ and $\Delta x_j$. In \cite{zhilu_2022_date}, the relation between $\Delta x_i$ and $\Delta x_j$ for a given $x_j$ are demonstrated, as shown in the first two diagrams in Fig.~\ref{fig:relu_distance}. If we consider that $x_j\in \mathbb{R}$, the relation between $\Delta x_i$ and $\Delta x_j$ will fall in the red shadow area as shown in the third diagram of Fig.~\ref{fig:relu_distance}.

Therefore, given an interval bound of $\Delta x_j$, written as $\Delta x_j \in [l, u]$,\footnote{We will discuss how we can derive this interval bound later. For now, let's assume this interval bound is known.} we can have a linear lower bound and a linear upper bound of $\Delta x_i$ w.r.t. $\Delta x_j$, as shown in the third diagram of Fig.~\ref{fig:relu_distance}. Formally, we can have the linear bounds in Eq.\eqref{equ:layer_lin_bound} for the ReLU layer by:
\begin{equation}
    \begin{aligned}
        \psi_j &= \frac{\max(u, 0)}{\max(u, 0) - \min(l,0)},\\
        \lambda &= -\frac{\max(u, 0)\min(l,0)}{\max(u, 0) - \min(l,0)},\\
        \omega_j &= -\frac{\min(l,0)}{\max(u, 0) - \min(l,0)}, \\
        \mu &= \frac{\max(u, 0)\min(l,0)}{\max(u, 0) - \min(l,0)}.
    \end{aligned}
\end{equation}

\subsubsection{Maxpooling Layer}

Another non-linear relation is the maxpooling layer. Consider one output element $s\in x_i$, we have 
\begin{equation} \label{equ:maxpool}
    s=\max\{t\mid t\in pool_s\},
\end{equation}
where $pool_s\subseteq x_j$ is the set of some elements of layer input. In that case, 
\[\Delta s=\max\{t + \Delta t\mid t\in pool_i\} - \max\{t\mid t\in pool_i\},\]
which is non-linear to $\Delta t\in pool_i$ and also related to the value of $t$. 

To enable linear symbolic propagation, the maxpooling layer can be converted into a combination of ReLU layers, addition layers, and subtraction layers. 

Consider the case that $pool_s = \{t, r\}$, where Eq.~\eqref{equ:maxpool} becomes a two-element maximization, we can have:
\[s=\max(t, r) = r + \max(t-r, 0) = r + relu(t-r).\]
In that case, we convert a two-element maxpooling into a subtraction, a ReLU and an addition. For maxpooling with more than two elements, we can also convert it into a nested two-element maxpooling. For instance $\max(t_1, t_2, t_3, t_4)=\max(\max(t_1, t_2),\max(t_3, t_4))$.

Through this conversion, a max-pooling layer can be converted into a series of linear layers and ReLU layers, where the functions to derive the layer bounds in Eq.~\eqref{equ:layer_lin_bound} are already presented.

\subsection{Neural network output violation bound}

When the layer bounds in Eq.~\eqref{equ:layer_lin_bound} for all layers are known, leveraging the backward propagation of the linear bounds, the linear bounds w.r.t. input perturbation in Eq.~\eqref{equ:final_lin_bound} can be derived. In this work, we consider the input perturbation is bounded by the L-$\infty$ norm: $\|\Delta x_{in}\|_\infty\leq \delta$. Combining with Eq.~\eqref{equ:final_lin_bound}, we will have
\begin{equation} \label{equ:final_interval_bound}
    \begin{aligned}
    -\|A\|_1 \delta + b \leq \mathcal{F}(x_{in} + \Delta x_{in}) - &\mathcal{F}(x_{in}) \leq \|C\|_1 \delta + d,\\
    & \forall x_{in},\ \forall \|\Delta x_{in}\|_\infty\leq \delta,
    \end{aligned}
\end{equation}
where $\|\cdot\|_1$ is the vector L-1 norm. 

\begin{theorem}
    The neural network is $(\delta, \varepsilon)$-globally robust if 
    \begin{equation}\label{equ:globrobust_condition}
        \max(|b -\|A\|_1 \delta|, |d + \|C\|_1 \delta|) \leq \varepsilon
    \end{equation}
\end{theorem}

\begin{proof}
    Omitted.
\end{proof}

Note that Eq.~\eqref{equ:globrobust_condition} is just a sufficient condition as there is over-approximation during the symbolic propagation, which, specifically, comes from the linear bounds of ReLU Layers.

\subsection{Input interval bounds of ReLU layers}

After converting maxpooling layers into a series of ReLU layers and linear layers, the only unknown parameters we need during the symbolic propagation are the interval bounds of the input of each ReLU layer. Previous, we assume such interval bounds are known. The approach to derive them is straightforward: Starting from the first ReLU layer, treat its input $x_j$ as the output of a sub-network, and evaluate its interval bound on this sub-network according to Eq.~\eqref{equ:final_lin_bound} and Eq.~\eqref{equ:final_interval_bound}. Then, evaluating the interval bounds of each subsequent ReLU layer, when the interval bounds are derived for all ReLU layers in the sub-network of which the ReLU input is the output.

\subsection{Interval bound refinement by branch-and-bound}

Once the interval bounds of all ReLU inputs are derived, the linear bounds of all layers as in Eq.~\eqref{equ:layer_lin_bound} can be calculated. Therefore, the global robustness of the neural network can be certified by using the symbolic propagation of the linear bounds, to derive the bounds w.r.t. network inputs, as in Eq.~\eqref{equ:final_lin_bound}. This symbolic propagation can be eventually written as a series of matrix operations as in Eq.~\eqref{equ:matrixOps}. These matrix operations can be deployed on GPUs to gain speedup just like the general training and inference of neural networks.

However, during the symbolic propagation of linear bounds, the linearization of the bounds of ReLU layers can cause a pessimistic over-approximation. Such over-approximation pessimism can be more and more significant with the increase of network size. To mitigate such pessimism, inspired by $\beta$-CROWN~\cite{beta_crown}, we propose an ad-hoc branch-and-bound approach for GROCET.

In this approach, we pick some ReLU nodes\footnote{Note that we here consider each individual node in the network, rather than each layer.} $x_i=relu(x_j)$ in the network, and branch them by the sign of $\Delta x_j$.\footnote{In $\beta$-CROWN~\cite{beta_crown}, it is branched by the sign of ReLU input $x_j$. For ReLU distance, rather than branch each ReLU input $x_j$ and $x_j + \Delta x_j$, we just branch the input distance $\Delta x_j$. We will explain our choice in the future.}

\begin{figure}[htb]
  \centering
  \includegraphics[width=\linewidth]{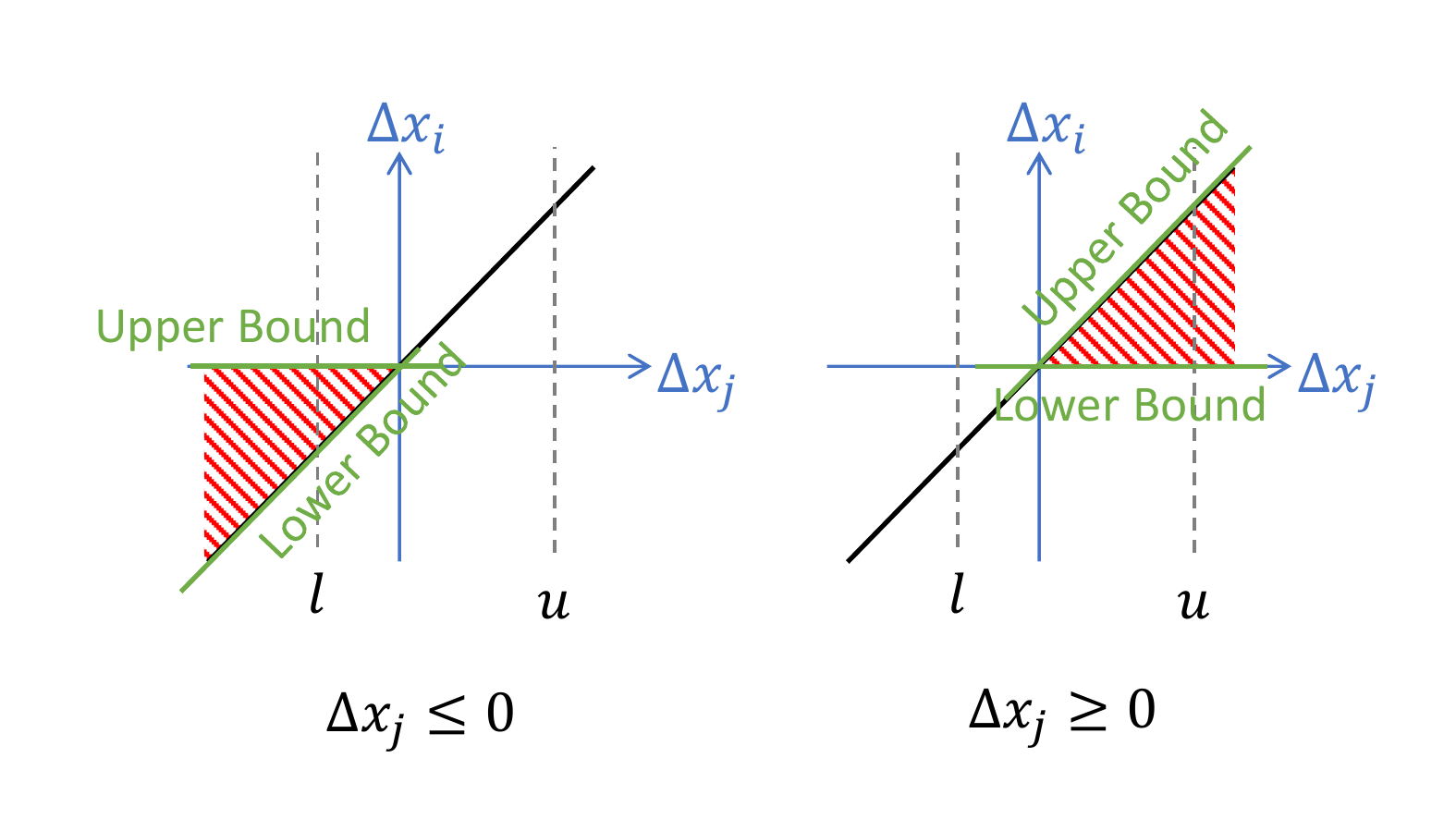}
  \caption{The linear lower and upper bounds of ReLU distance during the branch-and-bound of its input distance $\Delta x_j$.}
  \label{fig:relu_distance_bnb}
\end{figure}

Specifically, when branching a ReLU node $x_i=relu(x_j)$, we conduct the original symbolic propagation twice, under the cases of $\Delta x_j \leq 0$ and $\Delta x_j\geq 0$, respectively. Notice that by adding the constraint of $\Delta x_j \leq 0$ ($\Delta x_j\geq 0$), the relation between $\Delta x_j$ and $\Delta x_i$ will all fall into the lower-left red triangle (upper-right red triangle) area, as shown in Fig.~\ref{fig:relu_distance_bnb}. In these two symbolic propagation processes, the linear bounds of the ReLU layer are both tighter than the original bounds as shown in the third diagram of Fig.~\ref{fig:relu_distance}. As we split $\Delta x_j$ into two conditions, after the two symbolic propagation processes we can drive two interval bounds as in Eq.~\eqref{equ:final_interval_bound}. The final output variation bound will be the worst case among the two interval bounds.

To ensure the branching condition $\Delta x_j \leq 0$ (or $\Delta x_j\geq 0$) is satisfied during symbolic propagation, we leverage the same method as in $\beta$-CROWN~\cite{beta_crown}, which is the Lagrange multiplier. Take the lower output variation bound $A\Delta x_{in} + b$ as an example. We introduce a Lagrange multiplier $\beta_j\geq 0$ for each branched ReLU input distance $\Delta x_j$. For the case $\Delta x_j \leq 0$ and we have:
\begin{equation}
\begin{aligned}
    \Delta x_{out}|_{\Delta x_j \leq 0} &\geq \min_{\Delta x_{in}} \max_{\beta_j} A\Delta x_{in} + b + \beta_j S_j \Delta x_j\\
    &\geq \max_{\beta_j} \min_{\Delta x_{in}} A\Delta x_{in} + b + \beta_j S_j \Delta x_j
\end{aligned}
\end{equation}
where $S_j=1$.\footnote{for the condition $\Delta x_j \geq 0$, the Lagrange multiplier requires $S_j=-1$.} As The $\Delta x_j$ on the right-hand side can also be back-propagated to the input layer using the same symbolic propagation method, the objective of the inner minimization problem is still in the form of $A\Delta x_{in} + b$, i.e., linear to the input perturbation $\Delta x_{in}$. Therefore, for a given $\beta_j$ the inner minimization problem has a close form expression just like Eq.~\eqref{equ:final_interval_bound}. Similar to $\beta$-CROWN~\cite{beta_crown}, we solve the outer maximization problem by projected gradient descent (PGD)~\cite{madry2018towards} w.r.t. the Lagrange Multiplier $\beta_j$.

\section{Training Globally Robust Neural Networks}

The symbolic propagation of the linear bounds of neural network output variation, as presented in Section~\ref{sec:symb_prop}, not only provides the opportunity to leverage GPU to speedup the global robustness certification, but also makes it possible to evaluate the gradients of the global robustness (i.e., output variation bound) w.r.t. the network parameters (i.e., weights in each layer). Such gradients provide the opportunity to tune the network weights towards a more globally-robust neural network.

In this work, we propose a training framework that treats the neural network global robustness as a regularization term in the loss function to train globally-robust neural networks.

\subsection{Global robustness regularization}

Consider the commonly-used stochastic gradient descent (SGD) with mini-batch. Denote a batch of the training data as $(\textbf{x}, \textbf{y})$. The globally robust training loss is defined as:
\begin{equation}
    L(\mathcal{F}, \textbf{x}, \textbf{y}) = L_0(\mathcal{F}(\textbf{x}), \textbf{y}) + \lambda R_{GR}(\mathcal{F}),
\end{equation}
where $L_0$ is the regular loss term and $R_{GR}$ is the global robustness regularization. Specifically, we define $R_{GR}$ as:
\begin{equation}
    R_{GR} = (\|C\|_1 \delta + d) - (-\|A\|_1 \delta + b),
\end{equation}
which is the range of the neural network output variation bound interval in Eq.~\eqref{equ:final_interval_bound}. Note that to ensure the differentiability of the global robustness, the branch-and-bound technique in GROCET is not used during the training process.

As the global robustness is certificated for all possible input $\forall x_{in}$, the global robustness regularization term $R_{GR}$ is independent to the training sample $\textbf{x}$, which is the reason why it is a regularization term rather than a loss term. Therefore, the complexity of evaluating global robustness (as well as the evaluation of gradient during training), is irrelevant to the batch size of the training data.

\subsection{Gradient descent}
During gradient descent, the gradient of the loss function w.r.t. the network weights are needed. According to the neural network model considered in this work, the network weights $(W,b)$ will only occur in linear layers (as ReLU and maxpooling are separated layers). Moreover, as GROCET propagates the distances between the hidden neurons and their perturbed ones, the bias $b$ in each linear layer is not propagated. The gradient of $R_{GR}$ w.r.t. linear layer weight $W$, $\frac{\partial R_{GR}}{\partial W}$, can be derived by the chain rule according to the propagation of linear bounds as in Eq.~\eqref{equ:matrixOps}. We omit the detail here as it can be easily implemented in modern machine learning toolkits such as TensorFlow or PyTorch.

\section{Preliminary Experimental Results}

We evaluate GROCET for both global robustness certification and training. The algorithm is implemented in Python and TensorFlow is used to enable GPU acceleration. All experiments are conducted on a platform with Nvidia Titan RTX GPU. 

\begin{table*}[t]
\centering
\caption{Neural network setting, algorithm configuration and experimental results. The DNNs are trained for two datasets. Global robustness is evaluated under the input perturbation bound $\delta=2/255$ for MNIST DNNs and $\delta=0.001$ for CIFAR-10 DNNs.
For each DNN, the output variation bounds for 2 out of 10 outputs are presented.}
\begin{tabular}{|c|c|c|c|c|cccc|cc|c|}
\hline
\multirow{2}{*}{Dataset}  & \multirow{2}{*}{ID} & \multirow{2}{*}{Layers}                                                & \multirow{2}{*}{Neurons} & \multirow{2}{*}{$\underline{\varepsilon}$} & \multicolumn{4}{c|}{MILP-ITNE}                                                                                                                             & \multicolumn{2}{c|}{GROCET}                                                & \multirow{2}{*}{Improv.} \\ \cline{6-11}
                          &                     &                                                                        &                          &                                            & \multicolumn{1}{c|}{}                   & \multicolumn{1}{c|}{$r$}                 & \multicolumn{1}{c|}{$t$}                   & $\overline{\varepsilon}$ & \multicolumn{1}{c|}{$t$}                            & $\overline{\varepsilon}$ &                          \\ \hline
\multirow{6}{*}{MNIST}    & \multirow{2}{*}{1}  & \multirow{2}{*}{\begin{tabular}[c]{@{}c@{}}Conv:1\\ FC:2\end{tabular}} & \multirow{2}{*}{1416}    & 0.347                                      & \multicolumn{1}{c|}{\multirow{2}{*}{3}} & \multicolumn{1}{c|}{\multirow{2}{*}{30}} & \multicolumn{1}{c|}{\multirow{2}{*}{4.8h}} & 0.578                    & \multicolumn{1}{c|}{\multirow{2}{*}{\textbf{0.3h}}} & \textbf{0.446}           & 22.8\%                   \\ \cline{5-5} \cline{9-9} \cline{11-12} 
                          &                     &                                                                        &                          & 0.300                                      & \multicolumn{1}{c|}{}                   & \multicolumn{1}{c|}{}                    & \multicolumn{1}{c|}{}                      & 0.572                    & \multicolumn{1}{c|}{}                               & \textbf{0.398}           & 30.4\%                   \\ \cline{2-12} 
                          & \multirow{2}{*}{2}  & \multirow{2}{*}{\begin{tabular}[c]{@{}c@{}}Conv:2\\ FC:2\end{tabular}} & \multirow{2}{*}{3872}    & 0.453                                      & \multicolumn{1}{c|}{\multirow{2}{*}{3}} & \multicolumn{1}{c|}{\multirow{2}{*}{30}} & \multicolumn{1}{c|}{\multirow{2}{*}{3.3h}} & 0.874                    & \multicolumn{1}{c|}{\multirow{2}{*}{\textbf{1.3h}}} & \textbf{0.718}           & 17.8\%                   \\ \cline{5-5} \cline{9-9} \cline{11-12} 
                          &                     &                                                                        &                          & 0.420                                      & \multicolumn{1}{c|}{}                   & \multicolumn{1}{c|}{}                    & \multicolumn{1}{c|}{}                      & 0.723                    & \multicolumn{1}{c|}{}                               & \textbf{0.666}           & 7.9\%                    \\ \cline{2-12} 
                          & \multirow{2}{*}{3}  & \multirow{2}{*}{\begin{tabular}[c]{@{}c@{}}Conv:3\\ FC:2\end{tabular}} & \multirow{2}{*}{5824}    & 0.519                                      & \multicolumn{1}{c|}{\multirow{2}{*}{3}} & \multicolumn{1}{c|}{\multirow{2}{*}{30}} & \multicolumn{1}{c|}{\multirow{2}{*}{3.5h}} & 1.521                    & \multicolumn{1}{c|}{\multirow{2}{*}{\textbf{2.8h}}} & \textbf{1.412}           & 7.2\%                    \\ \cline{5-5} \cline{9-9} \cline{11-12} 
                          &                     &                                                                        &                          & 0.407                                      & \multicolumn{1}{c|}{}                   & \multicolumn{1}{c|}{}                    & \multicolumn{1}{c|}{}                      & 1.175                    & \multicolumn{1}{c|}{}                               & \textbf{1.106}           & 5.9\%                    \\ \hline
\multirow{4}{*}{CIFAR-10} & \multirow{2}{*}{4}  & \multirow{2}{*}{\begin{tabular}[c]{@{}c@{}}Conv:4\\ FC:2\end{tabular}} & \multirow{2}{*}{6352}    & 0.817                                      & \multicolumn{1}{c|}{\multirow{2}{*}{3}} & \multicolumn{1}{c|}{\multirow{2}{*}{0}}  & \multicolumn{1}{c|}{\multirow{2}{*}{5.7h}} & 291                      & \multicolumn{1}{c|}{\multirow{2}{*}{\textbf{1h}}}   & \textbf{261}             & 10.3\%                   \\ \cline{5-5} \cline{9-9} \cline{11-12} 
                          &                     &                                                                        &                          & 0.761                                      & \multicolumn{1}{c|}{}                   & \multicolumn{1}{c|}{}                    & \multicolumn{1}{c|}{}                      & 292                      & \multicolumn{1}{c|}{}                               & \textbf{261}             & 10.6\%                   \\ \cline{2-12} 
                          & \multirow{2}{*}{5}  & \multirow{2}{*}{\begin{tabular}[c]{@{}c@{}}Conv:4\\ FC:3\end{tabular}} & \multirow{2}{*}{10816}   & 1.136                                      & \multicolumn{1}{c|}{\multirow{2}{*}{3}} & \multicolumn{1}{c|}{\multirow{2}{*}{0}}  & \multicolumn{1}{c|}{\multirow{2}{*}{31h}}  & 1287                     & \multicolumn{1}{c|}{\multirow{2}{*}{\textbf{28h}}}  & \textbf{848}             & 34.1\%                   \\ \cline{5-5} \cline{9-9} \cline{11-12} 
                          &                     &                                                                        &                          & 1.074                                      & \multicolumn{1}{c|}{}                   & \multicolumn{1}{c|}{}                    & \multicolumn{1}{c|}{}                      & 1241                     & \multicolumn{1}{c|}{}                               & \textbf{867}             & 30.1\%                   \\ \hline
\end{tabular}
\label{tab:results}

{\raggedright ID - DNN ID; Layers - type and number of layers (include output layer); Neurons - total number of ReLU hidden neurons; $\underline{\varepsilon}$ - under-approximated output variation bound derived by PGD on the entire dataset; $w$ - network decomposition window size in MILP-ITNE; $r$ - number of refined neurons (each layer) in MILP-ITNE; $t$ - global robustness certification time; $\overline{\varepsilon}$ - over-approximated output variation bound derived by ITNE approaches; Improv. - improvement of GROCET on the tightness of global robustness over-approximation. \par}
\end{table*}

\begin{figure*}[htb]
  \centering
  \includegraphics[width=0.8\textwidth]{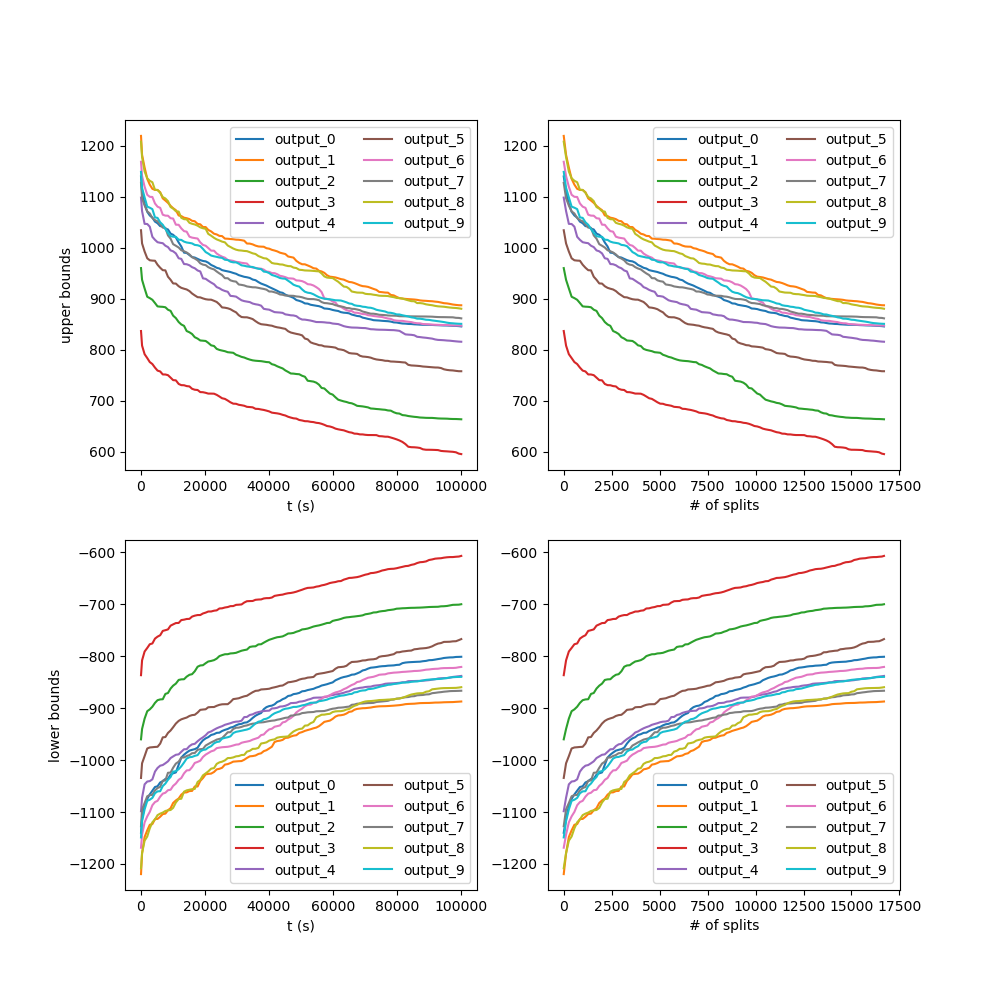}
  \caption{The ``current best'' over-approximated variation bounds of each output of DNN-5 during the certification process. The x-axis of left two figures is the certification time, while the x-axis of right two figures is the number of splits that GROCET has made during branch-and-bound.}
  \label{fig:cifar_progress}
\end{figure*}

\subsection{GROCET for global robustness certification}
Firstly, we compare the performance of GROCET with the previous optimization-based MILP-ITNE~\cite{zhilu_2022_date}. We conduct global robustness certification for several neural networks with different sizes. The networks are trained for MNIST and CIFAR-10 datasets. The neural networks range from 1400 hidden neurons to 10,000+ hidden neurons. The detail neural network configuration can be found in Table~\ref{tab:results}. Note that the MNIST networks (DNN-1 to DNN-2) are the same network used in~\cite{zhilu_2022_date}. 

The running time, as well as the global robustness certification results, are also listed in Table~\ref{tab:results}. Moreover, as a reference, we also include an under-approximation of the global robustness derived by the adversarial example searching method PGD~\cite{madry2018towards}.

From Table~\ref{tab:results}, we can observe that compared with MILP-ITNE, under similar running time, GROCET can achieve a much tighter over-approximation. And with similar level of over-approximation, GROCET can derive the output variation bounds at a much faster speed. All are achieved by the ability of GROCET to leverage the computation power of GPU. 

Moreover, as an approach based on branch-and-bound, GROCET can stop its optimization at any time and provide an instant result. The result will always be the ``current best'', which is the tightest over-approximation that it can find at that moment. In Fig.~\ref{fig:cifar_progress}, it shows the progress of GROCET that tightens the lower and upper bounds of the output variation of DNN-5. While MILP-CROWN can only have the certification result after it finishes solving all optimization problems (because of its network decomposition~\cite{zhilu_2022_date}), GROCET can use a timeout constraint to stop the process at any time the user wants.

\begin{table}[t]
    \centering
    \begin{tabular}{|c|c|c|c|c|}
    \hline
    Model& DNN-2 & DNN-2-GR & DNN-5 & DNN-5-GR\\
    \hline
    Train Acc & 95.35 & \textbf{96.90} & 70.99 & \textbf{87.82} \\ \hline
    Test Acc & 95.80 & \textbf{96.88} & \textbf{75.64} & 69.41 \\ \hline
    $\overline{\varepsilon}_{out1}$ & \textbf{0.718} & 0.746 & 848 & \textbf{55.8} \\ \hline
    $\overline{\varepsilon}_{out2}$ & 0.666 & \textbf{0.590} & 867 & \textbf{70.0} \\
    \hline
    \end{tabular}
    \caption{Comparison between regular training and globally-robust training on the MNIST model DNN-2 and CIFAR model DNN-5. DNN-2 and DNN-5 are from regular training while DNN-2-GR and DNN-5-GR are from globally robust training.}\label{tab:train_res}
\end{table}

\subsection{Training globally robust neural networks}

We leverage the differentiable property of GROCET-derived global robustness and train new models with the same model structure of DNN-2 (MNIST model) and DNN-5 (CIFAR model) by using global robustness as a regularization term in the loss function. Note that for the original DNN-2 and DNN-5 models, to avoid over-fitting, they are trained with L-2 regularization term in the loss function. The training and global robustness certification results are shown in Table~\ref{tab:train_res}. 

From~\ref{tab:train_res}, we can observe that for MNIST model DNN-2, with global robustness training, we can achieve higher accuracy while having a similar level of global robustness.  For CIFAR model DNN-5, we can observe that DNN-5-GR has a much higher training accuracy but lower testing accuracy. The possible reason is that when using global robustness as regularization rather than L-2 norm, the model is overfitted. However, on the other hand, on this larger DNN model, with global robustness training, the network global robustness can be more than 10x better than regular training with regularization.

\section{Conclusion}
In this paper, we propose an efficient GPU-support, differentiable global robustness certification tool GROCET. With GPU support, GROCET is much more efficient and more flexible than the state-of-the-art global robustness certification tool MILP-ITNE~\cite{zhilu_2022_date}. Moreover, leveraging the differentiable property, GROCET can facilitate the training of globally robust neural networks. Experimental results demonstrate that GROCET can help the training of neural networks with much higher global robustness, especially for large network models.


%





\ifCLASSOPTIONcaptionsoff
  \newpage
\fi


\bibliographystyle{IEEEtran}
\bibliography{IEEEabrv,zhilu.bib}

\end{document}